\documentclass[letterpaper,10pt,conference]{ieeeconf}

\IEEEoverridecommandlockouts
\usepackage{graphicx}

\usepackage{newtxtext}
\usepackage{newtxmath}

\usepackage{dsfont}
\usepackage{bm}
\usepackage{textcomp}
\usepackage{cite}

\usepackage{subcaption}
\usepackage{stfloats}
\usepackage{multirow}
\usepackage{soul}
\usepackage{makecell}
\usepackage{float}

\usepackage{algorithm}
\usepackage{algpseudocode}

\usepackage{verbatim}
\usepackage[table]{xcolor}

\usepackage{hyperref}
\usepackage{booktabs}

\algnewcommand\algorithmicinput{\textbf{\quad Input:}}
\algnewcommand\INPUT{\item[\algorithmicinput]}

\algnewcommand\algorithmicfind{\textbf{\qquad Find}}
\algnewcommand\FIND{\item[\algorithmicfind]}

\newtheorem{theorem}{Theorem}

\newtheorem{assumption}{Assumption}
\newtheorem{remark}{Remark}

\renewcommand{\eqref}[1]{Eqn.~(\ref{#1})}

\title{\LARGE \bf
Safe Multi-Robot Collaborative Transport Using Density Functions
}

\author{
Jagannath Prasad Sahoo,
Sriram S. K. S. Narayanan,
and Umesh Vaidya
\thanks{ 
Jagannath Prasad Sahoo is with the Center for Artificial Intelligence and
Robotics, Indian Institute of Technology Mandi, India
{\tt\small email: s23118@students.iitmandi.ac.in}.
Sriram S. K. S. Narayanan and Umesh Vaidya are with the Department of
Mechanical Engineering, Clemson University, Clemson, SC 29630, USA
{\tt\small email: sriramk@clemson.edu, uvaidya@clemson.edu}.
}
}

\begin{document}

\maketitle
\thispagestyle{empty}
\pagestyle{empty}


\begin{abstract}
This paper presents a hierarchical density-based model predictive control
framework for safe collaborative manipulation by multiple quadrupedal robots.
The framework enables a team of robots to push a shared object to a desired pose
using only the initial and goal poses, without requiring a precomputed reference
trajectory. A centralized box-level MPC optimizes contact forces while
enforcing a control-density constraint for goal convergence and obstacle
avoidance. Each robot then solves its own distributed robot-level
whole-body MPC, under a stated shared-information assumption, to track
its moving contact location while accounting for static obstacles and the
time-varying positions of neighboring robots. The
approach is evaluated in MuJoCo using whole-body contact dynamics for two and
three Unitree Go2 quadrupeds collaboratively pushing rigid objects through
narrow passages. Comparisons with matched Control Barrier Function and
RRT* based tracking baselines demonstrate the effectiveness of the proposed
density-based formulation for push-only, force- and torque-coupled
manipulation tasks. Implementation videos are available at
\url{https://jaggu2606.github.io/go2-density-mpc-pushing/}.
\end{abstract}


\section{Introduction}\label{sec:intro}


Quadrupedal teams can support exploration, inspection, and search-and-rescue
\cite{chung2023depths,penaqueralta2020sar}, where coordinated,
collision-free motion is essential in cluttered environments
\cite{oh2015survey,hsu2024safetyfilter}. Tasks such as debris removal
and material transport may also require teams to manipulate objects
beyond a single robot's capacity. Collaborative
pushing requires the robots to coordinate their contact forces while
avoiding collisions with one another, the object, and surrounding
obstacles \cite{sombolestan2023hierarchical}. This is especially
challenging in push-only manipulation, where the contact forces couple
the object's translation and rotation.

Collaborative pushing has been approached through planning, optimization-based
control, and learned coordination. Sampling-based planners such as RRT*
generate collision-free paths, but these plans must still be converted into
feasible object motions and robot actions
\cite{karaman2011sampling}. For contact-rich loco-manipulation, a hierarchical
MPC for a single quadruped jointly optimizes contact forces and configuration
while coordinating with whole-body locomotion control, using a prescribed
object trajectory \cite{rigo2023contact}. Related methods address uncertain
object mass and friction through adaptive force regulation
\cite{sombolestan2023hierarchical}, and unmodeled payloads through
CLF-constrained MPC \cite{minniti2022adaptive}. At the team level,
hierarchical multi-agent reinforcement learning has been used to coordinate
quadruped pushing through high-level subgoals and decentralized execution
\cite{feng2025learning}, and collaborative pushing scenarios have been
included in multi-agent quadruped benchmarks \cite{xiong2024mqe}. These approaches demonstrate coordinated pushing, but do not provide
formal guarantees of collision safety for the robots and shared object
during execution.

Formal safety constraints have also been incorporated into robot control.
Control barrier functions provide conditions for forward invariance of safe
sets and have been used in MPC \cite{zeng2021mpccbf}; distributed
barrier-based controllers further address safety and convergence in
multi-agent navigation \cite{mestres2024distributed}. These methods offer
formal guarantees under their stated assumptions, but do not address the
coupled problem of optimizing multiple push-only contact forces on a shared
object while coordinating robot-level motion around moving contact targets
and neighboring robots.

Density functions provide a complementary dual-space formulation
of safety and convergence, based on conditions on the evolution of
trajectory density \cite{rantzer2004safe}. Related ideas have been applied
to formation stabilization \cite{dimarogonas2007application}, safe
navigation and data-driven planning
\cite{loizou2008density,vaidya2018optimal,yu2022datadriven}, and learned
transfer-operator models for off-road and legged locomotion
\cite{moyalan2023convex,moyalan2023offroad}. Control density functions have been used to formulate convex and robust
safety constraints for quadratic-program controllers
\cite{moyalan2024synthesizing, moyalan2024robustcdf}, and have been applied to safe motion planning and control of robotic systems \cite{zheng2023safe,narayanan2023safe,narayanan2025density}.
These formulations are primarily developed for continuous-time systems.
More recently, control density functions have been incorporated into MPC
with tunable safety margins \cite{narayanan2025mpccdf}. However, existing density-based safety methods have focused primarily on
single-agent systems and do not directly address predictive safety in multi-agent settings.


This paper presents a hierarchical density-constrained MPC framework for
collaborative manipulation by multiple quadrupedal robots, consisting of a
single centralized box-level MPC coupled with $N$ distributed robot-level
MPCs. A box-level MPC
selects the contact forces that drive the shared object toward a desired pose
while accounting for static obstacles. Each robot then solves its own
robot-level whole-body MPC locally to track
its moving contact location and enforce additional safety constraints with
respect to obstacles and neighboring robots, under the shared-information
assumption stated in Section~\ref{sec:prelims}. The main contributions are:
{\color{black}
\begin{itemize}
    \item A hierarchical multi-agent MPC framework for collaborative
    pushing that couples a centralized object-level MPC for push-only
    contact-force optimization with distributed robot-level MPCs for
    tracking moving contact locations, enabling goal-directed transport
    without a prescribed object trajectory.

    \item A discrete-time safety constraint based on density functions,
    established in Theorem~\ref{thm:discrete_density_condition} and
    integrated into both MPC layers to enforce safe predictive control
    for obstacle avoidance and inter-agent separation.

    \item MuJoCo validation of collaborative transport with two and
    three Unitree Go2 quadruped robots, including comparisons against matched
    CBF-MPC~\cite{zeng2021mpccbf} and
    RRT* based tracking baselines~\cite{karaman2011sampling}.
\end{itemize}
}

\begin{figure*}
\centering
\includegraphics[width=0.9\linewidth]{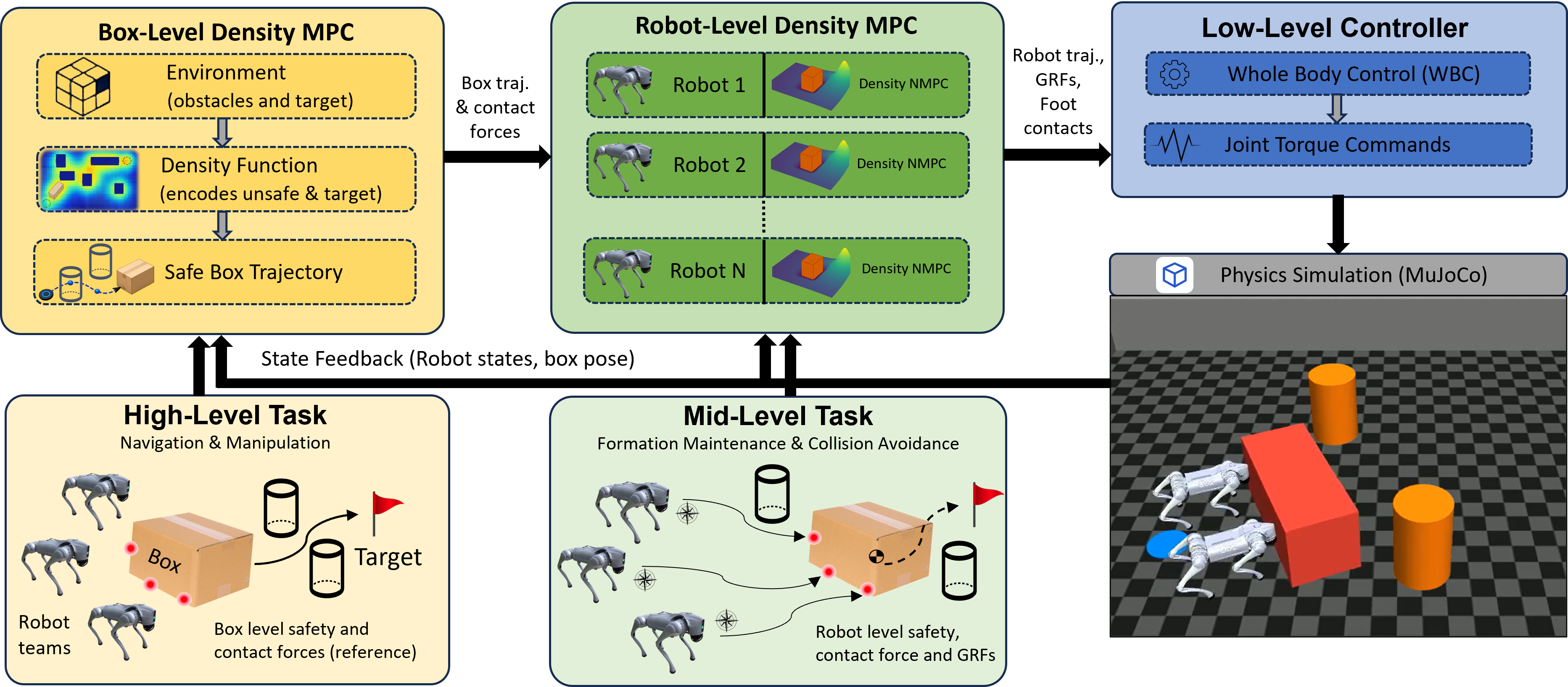}
\caption{Hierarchical architecture of the proposed framework with box-level density MPC and robot-level density MPC.}
\label{fig:architecture}
\end{figure*}


\section{Preliminaries and Problem Formulation}\label{sec:prelims}

\subsection{Notation and Discrete-Time Dynamics}
Let $\mathbb{R}^n$ denote the $n$-dimensional Euclidean space and let
$\mathcal{X}\subset\mathbb{R}^n$ denote a bounded state space. The state and
input of a discrete-time dynamical system are denoted by
$x\in\mathcal{X}$ and $u\in\mathcal{U}\subset\mathbb{R}^m$, respectively. Let $\mathcal{X}_0$, $\mathcal{X}_T$, and $\mathcal{X}_u$ denote the initial,
target, and unsafe sets. The safe set is defined as $\mathcal{X}_s
:=
\mathcal{X}\setminus\mathcal{X}_u.$ The spaces $\mathcal{L}_1(\mathcal{X})$ and
$\mathcal{C}^k(\mathcal{X})$ denote the spaces of integrable functions and
$k$-times continuously differentiable functions on $\mathcal{X}$,
respectively. The notation $m(\mathcal{A})$ denotes the Lebesgue measure of a
measurable set $\mathcal{A}\subset\mathcal{X}$, and
$\mathds{1}_{\mathcal{A}}(x)$ denotes its indicator function. We use
$\mathcal{M}(\mathcal{X})$ to denote the space of measures on $\mathcal{X}$. Let $\mathcal{B}_{\delta}(\mathcal{X}_T)$ denote a $\delta$-neighborhood of
the target set. We define $\mathcal{X}_1
:=
\mathcal{X}\setminus\mathcal{B}_{\delta}(\mathcal{X}_T)$. The forward-difference operator is defined as
$\Delta\rho(x(k),u(k)):=\rho(x(k+1))-\rho(x(k))$.
For a differentiable map $F$, $DF(x)$ denotes its Jacobian
matrix at $x$, and $\det(\cdot)$ denotes the determinant.
The symbols $|\cdot|$ and $\|\cdot\|$ denote the absolute
value and Euclidean norm, respectively

Consider the discrete-time system
\begin{equation}
x(k+1)
=
F(x(k),u(k))
=:
F_{u(k)}(x(k)),
\label{eq:generalDiscreteDynamics}
\end{equation}
where $F:\mathcal{X}\times\mathcal{U}
\rightarrow
\mathcal{X}$ is a continuous state-transition map.

\begin{assumption}[Input-Wise Invertibility]\label{as:invertibility}
For each individually fixed value $u\in\mathcal{U}$ 
the system mapping $F_u:=F(\cdot,u)$ is assumed to be invertible and hence nonsingular on $\cal X$.
\end{assumption}

Assumption~\ref{as:invertibility} is a regularity property of the
dynamics $F$ itself, required below to pull a density back through
$F_u$. The input actually applied to the system, $u(k)$, is in general
time-varying and is generated in closed loop by the receding-horizon
controllers of Section~\ref{sec:density_mpc}.

\subsection{Perron--Frobenius Operator}
Under Assumption~\ref{as:invertibility}, for an invertible differentiable
map $F_u$ evaluated at a particular $u\in\mathcal{U}$ as in
\eqref{eq:generalDiscreteDynamics}, the
Perron--Frobenius (P-F) operator propagates densities
$\varphi\in\mathcal{L}_1(\mathcal{X})$ as
\begin{equation}
[\mathbb{P}_u\varphi](x)
=\varphi(F_u^{-1}(x))
\left|\det D F_u^{-1}(x)\right|.\label{PF}
\end{equation}
where $\varphi$ is the state-space density.
The absolute determinant accounts for local volume changes. The Jacobian determinant accounts for changes in state-space
volume. This density-transport perspective underlies
occupancy-based formulations of safe navigation
\cite{rantzer2004safe,vaidya2018optimal}. In the MPC realization used
here, $u$ is not held fixed over time: $\mathbb{P}_u$ is applied
pointwise, at the particular input realized at each prediction stage, to
obtain the stage-wise density condition \eqref{eq:densitySafetyConstraint}
enforced inside the receding-horizon problems \eqref{eq:objectLevelMPC}
and \eqref{eq:robotLevelMPC}. Because that condition is re-imposed at
every stage of every solve, it constrains the actual, time-varying
closed-loop trajectory, not a hypothetical constant-input one.

\subsection{Safe Set}
Let $h_j:\mathcal{X}\rightarrow\mathbb{R}$ be a continuous safety function
associated with obstacle $j$, where $j=1,\ldots,L$. The individual safe sets
are defined as $\mathcal{S}_j
=
\left\{
x\in\mathcal{X}:h_j(x)\geq0
\right\}.$ The overall safe and unsafe sets are then
\begin{equation*}
\mathcal{S}
=
\bigcap_{j=1}^{L}\mathcal{S}_j,
\qquad
\mathcal{X}_u
=
\mathcal{X}\setminus\mathcal{S}.
\end{equation*}

Equivalently, $\mathcal{S}
=
\left\{
x\in\mathcal{X}:h_j(x)\geq0,\;
j=1,\ldots,L
\right\}.$ The safety functions used at the object and robot levels are constructed
from their respective safety requirements.

\subsection{Control Density Functions}

Control
density functions characterize safe navigation through the
transport of density \cite{rantzer2004safe,moyalan2024robustcdf}.
We use $\rho(x)=\Phi(x)/V(x)^\alpha$, $\alpha>0$, where
$V$ measures goal error and the nonnegative bump function
$\Phi$ vanishes on the unsafe set. Under the regularity and
density-divergence conditions established in
\cite{moyalan2024robustcdf}, this framework provides
almost-everywhere safety and convergence, allowing exceptions
only for a set of initial conditions of zero Lebesgue measure.

\subsection{Box and Robot Dynamics}

Consider a rigid box manipulated by a team of $N$ quadrupedal robots. The box
position and heading are denoted by
$p_b=[p_{b,x},p_{b,y}]^\top\in\mathbb{R}^2$ and
$\psi_b\in\mathbb{R}$, respectively. The corresponding translational and
angular velocities are denoted by
$v_b=[v_{b,x},v_{b,y}]^\top$ and $\omega_b$. The box state is $x_b
=
\begin{bmatrix}
\psi_b &
p_{b,x} &
p_{b,y} &
\omega_b &
v_{b,x} &
v_{b,y}
\end{bmatrix}^{\top}
\in\mathcal{X}_b\subset\mathbb{R}^{6}.$ Each robot applies a nonnegative force along the box body-frame push
direction. The box dynamics are modeled as
\begin{equation}
\begin{aligned}
\dot p_b &= v_b,
\qquad m_b\dot v_b
=
\sum_{i=1}^{N}
R(\psi_b)
\begin{bmatrix}
f_{c,i}\\
0
\end{bmatrix}
+
f_{\mu}(v_b),\\
\dot\psi_b &= \omega_b, \qquad I_b\dot\omega_b
=
-\sum_{i=1}^{N}f_{c,i}d_i,
\end{aligned}
\label{eq:boxDynamics}
\end{equation}
where $m_b$ and $I_b$ are the box mass and planar moment of inertia,
$R(\psi_b)\in SO(2)$ is the planar rotation matrix,
$f_{\mu}(v_b)$ represents the box-ground friction force, and
$d_i$ is the fixed lateral contact offset of robot $i$ from the box
centerline.
The contact-force constraints are $0\leq f_{c,i}\leq f_{c,\max},\,i\in\mathcal{I}_N,$ where $\mathcal{I}_N:=\{1,\ldots,N\}$. After discretization with sampling time $T_s$, the box dynamics are written
as
\begin{equation}
x_b(k+1)
=
F_b\big(x_b(k),f_c(k);d_{1:N}\big),
\label{eq:boxDiscreteDynamics}
\end{equation}
where $f_c(k)
=
\begin{bmatrix}
f_{c,1}(k)&\cdots&f_{c,N}(k)
\end{bmatrix}^{\top}.$ For robot $i\in\mathcal{I}_N$, let
$x_i\in\mathcal{X}_i\subset\mathbb{R}^{n_i}$ denote its whole-body state
and let $u_i\in\mathcal{U}_i\subset\mathbb{R}^{m_i}$ denote its control
input. The robot dynamics are
\begin{equation}
x_i(k+1)
=
F_i\big(x_i(k),u_i(k)\big),
\qquad
u_i(k)\in\mathcal{U}_i.
\label{eq:agentDynamics}
\end{equation}

The contact location assigned to robot $i$ is determined by the current box
state:
\begin{equation}
p_{c,i}(x_b)
=
p_b+
R(\psi_b)
\begin{bmatrix}
0\\
d_i
\end{bmatrix}.
\label{eq:contactPoint}
\end{equation}

Because $d_i$ is fixed, the contact location moves with the box but does not
change independently during the maneuver.

\subsection{Problem Statement}

Consider $N$ robotic agents collaboratively manipulating a shared object in an
environment with static obstacles. Let $\xi_b(k)
=
\begin{bmatrix}
p_b(k)^\top & \psi_b(k)
\end{bmatrix}^{\top}$ denote the object pose, and let
$\xi_b^\star=[(p_b^\star)^\top\ \psi_b^\star]^\top$ be the desired pose.

Given the object and robot dynamics in
\eqref{eq:boxDiscreteDynamics} and \eqref{eq:agentDynamics}, the initial
object and robot states, the obstacle geometry, and the contact geometry,
design object-level contact forces $f_c(k)$ and robot-level feedback laws
\begin{equation}
u_i(k)
=
\pi_i\Big(
x_i(k),x_b(k),\{x_j(k)\}_{j\neq i}
\Big),
\qquad
i\in\mathcal{I}_N,
\label{eq:controlLaw}
\end{equation}
such that the agents collaboratively transport the object to its desired pose
while satisfying
\begin{enumerate}
    \item object convergence: $\xi_b(k)\longrightarrow\xi_b^\star;$
    \item object- and agent-level safety, namely
    $x_b(k)\in\mathcal{S}_b$ and
    $x_i(k)\in\mathcal{S}_i$ for all $i\in\mathcal{I}_N$, while each robot
    tracks its assigned moving contact location $p_{c,i}(x_b)$; and

    \item admissible inputs: $f_c(k)\in\mathcal{F}_c$, $u_i(k)\in\mathcal{U}_i$, $i\in\mathcal{I}_N$.
\end{enumerate}

The proposed hierarchical controller addresses these objectives through an
object-level density MPC for contact-force regulation and a robot-level
density MPC for moving contact-point tracking and obstacle and inter-agent
avoidance.

The feedback law \eqref{eq:controlLaw} lets robot $i$ depend on the states
$\{x_j(k)\}_{j\neq i}$ of every other robot. The next assumption states
precisely how that dependence is realized and fixes the resulting
computational architecture.

\begin{assumption}[Information Availability]\label{as:info}
At each sampling instant $k$, the object state $x_b(k)$ and every
robot state $x_j(k)$, $j\in\mathcal{I}_N$, are available from a single
shared state, so that robot $i$'s controller receives its own state
$x_i(k)$, the object-level MPC's contact-force reference $f_c(k)$ and
predicted contact target $p_{c,i}(x_b)$, and $\{x_j(k)\}_{j\neq i}$, with
zero delay and zero estimation error.
\end{assumption}

Under Assumption~\ref{as:info}, the object-level problem
\eqref{eq:objectLevelMPC} is solved once, centrally, over the joint
contact-force vector $f_c(k)$, while the $N$ robot-level problems
\eqref{eq:robotLevelMPC} are each solved independently, one per robot,
using only the locally available information above. The resulting
architecture is centralized at the object level and distributed at the
robot level.

\begin{remark}[Communication and Consensus]\label{rem:consensus}
\textcolor{black}{The robot-level MPCs are distributed in computation
and coordinated through contact-force and motion references supplied
by the centralized object-level MPC. Since each robot tracks a distinct
contact target, an explicit consensus protocol is not required.
Assumption~\ref{as:info} supplies the teammate states needed for local
prediction and inter-agent constraints, realized here through shared
simulator states. Future work will explore consensus-based state
estimation over local communication networks
\cite{olfati2004consensus,ren2008distributed} and robust density
constraints to accommodate estimation and communication uncertainty
\cite{moyalan2024robustcdf}.}
\end{remark}

\section{Construction of Density Functions}
\label{sec:density_construction}

This section summarizes the density construction used for object-level and
robot-level safety. Following
\cite{moyalan2024robustcdf,narayanan2025mpccdf,rantzer2004safe}, each unsafe
region is represented using an implicit function, converted into a smooth
bump function, and combined with a goal-error function to form a control
density.

\subsection{Unsafe Sets and Smooth Bump Functions}

Let $q\in\mathbb{R}^{n_q}$ denote the position or configuration variables
used for collision checking. An unsafe region associated with obstacle $j$ is
defined by $\mathcal{X}_{u,j}=\left\{q:h_j(q)\leq0\right\},$
where $h_j:\mathbb{R}^{n_q}\rightarrow\mathbb{R}$ is a continuous implicit
safety function. A sensing region is defined using
$s_j:\mathbb{R}^{n_q}\rightarrow\mathbb{R}$ as $\mathcal{X}_{s,j}=\left\{q:s_j(q)\leq0\right\}\setminus\mathcal{X}_{u,j}.$ Following the smooth construction in
\cite{zheng2023safe}, define
\begin{equation*}
f(\tau)
=
\begin{cases}
\exp\left(-\dfrac{1}{\tau}\right), & \tau>0,\\[6pt]
0, & \tau\leq0,
\end{cases}
\label{eq:elementaryFunction}
\end{equation*}
and
\begin{equation*}
\bar{f}(\tau)
=
\frac{f(\tau)}
{f(\tau)+f(1-\tau)}.
\label{eq:smoothStep}
\end{equation*}

The bump function associated with obstacle $j$ is
\begin{equation}
\Psi_j(q)
=
\begin{cases}
0,
& h_j(q)\leq0,\\[6pt]
\displaystyle
\bar{f}\left(
\frac{h_j(q)}
{h_j(q)-s_j(q)}
\right),
& h_j(q)>0,\;s_j(q)\leq0,\\[12pt]
1,
& s_j(q)>0.
\end{cases}
\label{eq:bump}
\end{equation}

Thus, $\Psi_j$ is zero in the unsafe region, transitions smoothly from zero
to one in the sensing region, and equals one outside the sensing region. For
$L$ unsafe regions,
\begin{equation}
\Phi(q)
=
\prod_{j=1}^{L}\Psi_j(q).
\label{eq:aggbump}
\end{equation}

A density function is then constructed as
\begin{equation}
\rho(x)
=
\frac{\Phi(q(x))}
{V(x)^{\alpha}},
\qquad
\alpha>0,
\label{eq:densityFunction}
\end{equation}
where $V(x)$ measures the error relative to the desired target. This density construction has previously been used to synthesize safety controllers by enforcing a continuous-time density-divergence condition, including the safety-QP formulation in \cite{moyalan2024synthesizing}. In particular, we have the following theorem.
\begin{theorem}[\cite{moyalan2024synthesizing}]
\label{theorem_density} For a continuous time control dynamical system $\dot x=f(x,u)$, the almost everywhere safe navigation problem of steering the system from the initial set to final set while avoiding the unsafe set ${\cal X}_u$ if there exists $u=u(x)$ such that $\nabla\cdot(f(x,u)\rho)>0$. 
\end{theorem}

In this work, we establish a discrete-time safety condition in Theorem~\ref{thm:discrete_density_condition} and incorporate it into the proposed density-MPC framework for collaborative transport.

\begin{figure}
    \centering
    \includegraphics[width=0.8\linewidth]{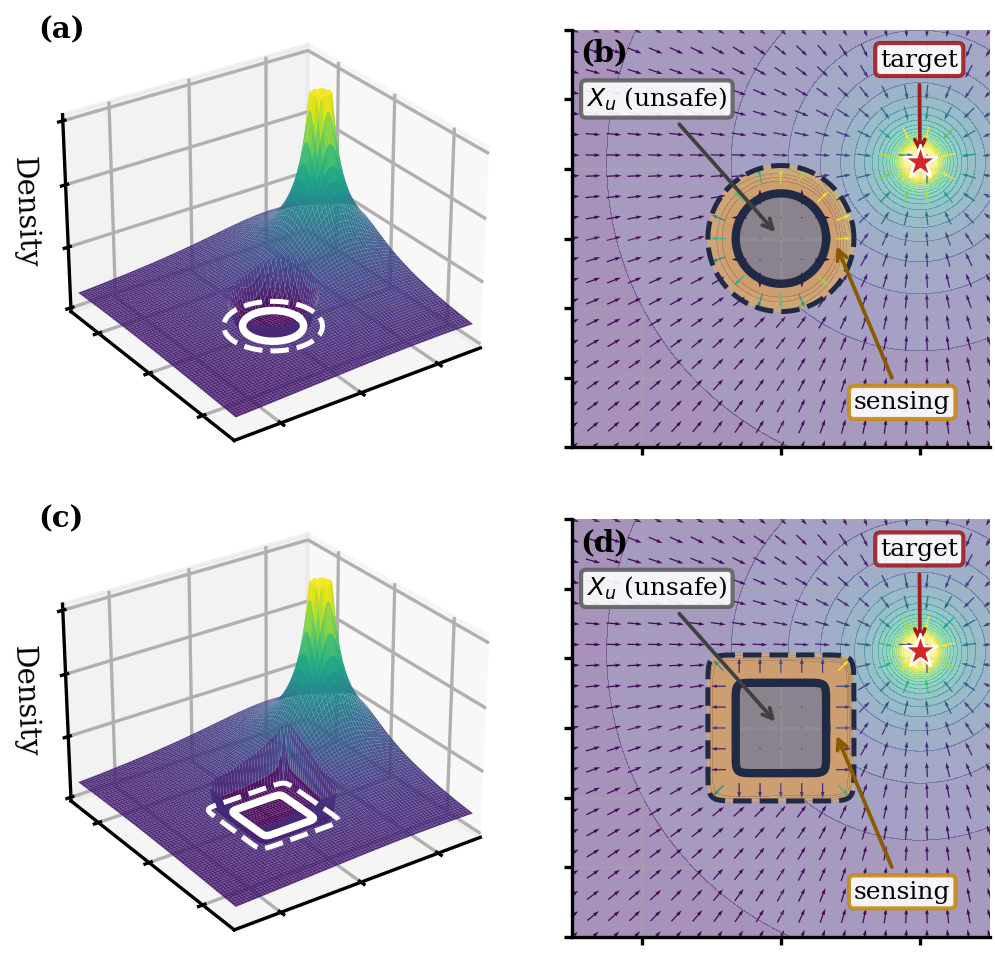}
    \caption{Panels (a) and (b) show the density and corresponding navigation field for a circular region, (c) and (d) show the corresponding construction for a rectangular region.}
\label{fig:geometric_density}
    \label{fig:densit_figure}
\end{figure}



\subsection{Object-Level Density}

At the object level, obstacle geometry is defined with respect to the complete
object footprint. Let $\Phi_b(p_b)$ denote the aggregate bump function
constructed from the footprint-aware obstacle representations. The
object-level density is
\begin{equation}
\rho_b(x_b)
=
\frac{\Phi_b(p_b)}
{V_b(x_b)^{\alpha_b}},
\label{eq:boxDensityConstruction}
\end{equation}
where $V_b(x_b)
=
\|p_b-p_b^\star\|^2
+
w_\theta(\psi_b-\psi_b^\star)^2.$ Here, $(p_b^\star,\psi_b^\star)$ is the desired object pose and
$w_\theta>0$ weights the orientation error. The object-level density
therefore combines goal-directed motion with obstacle avoidance. A cuboidal unsafe region is approximated by the sublevel set
$\{q\in\mathbb{R}^3:h_{\mathrm{cuboid}}(q)\leq0\}$, where
\begin{equation}
h_{\mathrm{cuboid}}(q)
=
\sum_{\ell=1}^{3}
\left|\frac{q_\ell-c_\ell}{a_\ell}\right|^{p}-1.
\label{eq:geometry_examples}
\end{equation}
Here, $c_\ell$ and $a_\ell>0$ denote the center coordinate
and half-length along axis $\ell$, respectively.
Larger even values of $p$ yield a closer approximation
to an axis-aligned cuboid. Figure~\ref{fig:geometric_density}(c)--(d) illustrates the
density and corresponding navigation field for a
two-dimensional rectangular region.


\subsection{Robot-Level Multi-Agent Density}

For robot $i$, let $q_i=\Pi_i x_i$ denote the configuration variables used
for collision checking. Pairwise robot safety is described by $h_{ij}(q_i,q_j) = \|q_i-q_j\|^2-d_{ij}^2,$
where $d_{ij}>0$ is the desired separation between robots $i$ and $j$.
More general pairwise functions can be used for ellipsoidal, rectangular, or
configuration-dependent robot representations.

For a fixed configuration of the neighboring robots, the agent-induced unsafe
set for robot $i$ is
\begin{equation}
\mathcal{X}_{u,i}^{\mathrm{agent}}(k)
=
\bigcup_{j\neq i}
\left\{
q_i:
h_{ij}\big(q_i,q_j(k)\big)<0
\right\}.
\label{eq:agentUnsafeSet}
\end{equation}

Including environmental obstacles, the complete robot-level unsafe set is
\begin{equation}
\mathcal{X}_{u,i}(k)
=
\mathcal{O}_i(k)
\cup
\mathcal{X}_{u,i}^{\mathrm{agent}}(k).
\label{eq:unsafeSet}
\end{equation}

The robot-level density uses the moving contact location as its target:
\begin{equation}
\rho_i(x_i,x_b;k)
=
\frac{\Phi_i(p_i;k)}
{V_i(x_i,x_b)^{\alpha_i}},
\label{eq:robotDensityConstruction}
\end{equation}
where $V_i(x_i,x_b)
=
\|p_i-p_{c,i}(x_b)\|^2.$ The bump function $\Phi_i(p_i;k)$ includes both environmental obstacles and
the time-varying safety envelopes of neighboring robots. Its parameters are
updated online as the multi-agent configuration changes. 

\begin{remark}[Planar Robot-to-Robot Safety]
\label{rem:planarSafety}
In this work, quadrupedal navigation and inter-agent collision checking are
performed in the planar workspace. Each neighboring robot is therefore
represented by a circular safety envelope
\begin{equation*}
\mathcal{B}_j(x_j(k))
=
\left\{
q_i\in\mathbb{R}^2:
\|q_i-p_j(k)\|^2\leq r_{ij}^2
\right\},
\end{equation*}
where $p_j(k)$ is the planar base position of robot $j$ and $r_{ij}$ includes the robot footprint and safety margin.
\end{remark}


\section{Safety Using Discrete-Time Density Constraints}
\label{sec:density_mpc}
Existing density-based safety results, including \cite{moyalan2024robustcdf,moyalan2023convex}, are developed for continuous-time systems and use the divergence condition $\nabla\cdot (f(x,u)\rho(x))> 0$ (Theorem \ref{theorem_density}), with $\rho(x)$ as in \eqref{eq:densityFunction}. Since the NMPC formulation here is discrete-time, Theorem~\ref{thm:discrete_density_condition} derives the corresponding discrete-time condition from this continuous-time divergence condition. 
\begin{theorem}[Discrete-time density safety condition]
\label{thm:discrete_density_condition}
Consider the continuous-time control system
$   \dot{x}=f(x,u),$ and let $\rho\in C^2(\mathcal X)$ be a nonnegative density satisfying $    \nabla\cdot\big(\rho(x)f(x,u)\big)>0.
    \label{eq:continuous_density_condition}
$ Consider its Euler discretization with sampling time $T_s>0$,
$    x^+=F_u(x):=x+T_s f(x,u).
    \label{eq:euler_discretization}
$
Suppose that $F_u$ is nonsingular on the domain of interest. Then,
for sufficiently small $T_s$,
\begin{equation}
    \rho(F_u(x))
    \left|\det DF_u(x)\right|-\rho(x)>0.
    \label{eq:discrete_density_jacobian}
\end{equation}
Equivalently, using the Perron--Frobenius operator $\mathbb P_u$
associated with $F_u$,
$    \rho(y)-[\mathbb P_u\rho](y)>0,
    \;\; y=F_u(x).
    \label{eq:discrete_density_PF}
$
\end{theorem}

\begin{proof}
For the Euler map $ x^+=F_u(x):=x+T_s f(x,u)$, we have $
    DF_u(x)=I+T_sD_xf(x,u)$. 
For sufficiently small $T_s$,
\begin{equation}
    \det DF_u(x)
    =
    1+T_s\nabla\cdot f(x,u)+O(T_s^2).
    \label{eq:jacobian_expansion}
\end{equation}
Similarly, Taylor expansion of the density gives
\begin{equation}
    \rho(F_u(x))
    =
    \rho(x)
    +T_s\nabla\rho(x)^\top f(x,u)
    +O(T_s^2).
    \label{eq:density_expansion}
\end{equation}
Multiplying \eqref{eq:jacobian_expansion} and
\eqref{eq:density_expansion} yields
\begin{align}
&\rho(F_u(x))\left|\det DF_u(x)\right|-\rho(x)
\nonumber\\
&\qquad =
T_s\left[
    \nabla\rho(x)^\top f(x,u)
    +\rho(x)\nabla\cdot f(x,u)
\right]
+O(T_s^2)
\nonumber\\
&\qquad =
T_s\nabla\cdot\big(\rho(x)f(x,u)\big)
+O(T_s^2).
\label{eq:density_expansion_final}
\end{align}
Hence, under $\nabla\cdot (\rho(x) f(x,u))>0$ the
right-hand side is positive for sufficiently small $T_s$, establishing
\eqref{eq:discrete_density_jacobian}.
Now let $y=F_u(x)$. By the definition of the Perron--Frobenius
operator,
\begin{align}
    [\mathbb P_u\rho](y)
    =
    \rho(F_u^{-1}(y))
    \left|\det DF_u^{-1}(y)\right| =
    \frac{\rho(x)}
    {\left|\det DF_u(x)\right|}.
\end{align}
Dividing \eqref{eq:discrete_density_jacobian} by
$\left|\det DF_u(x)\right|>0$ gives $
    \rho(y)-[\mathbb P_u\rho](y)>0,
$
which proves the result.
\end{proof}

Motivated by Theorem~\ref{thm:discrete_density_condition}, define
the discrete-time density constraint used in the MPC as
\begin{equation}
    \mathcal D_\rho(x^+,x,u)
    :=
    \rho(x^+)-[\mathbb P_u\rho](x^+)\geq0.
    \label{eq:densitySafetyConstraint}
\end{equation}
At each prediction stage $j$, we write the density constraint as 
\begin{align*}
\mathcal D_{b,j}&:=
\mathcal D_{\rho_b}
(x_{b,k+j+1|k},x_{b,k+j|k},f_{c,k+j|k}),\\
\mathcal D_{i,j}&:=
\mathcal D_{\rho_i}
(x_{i,k+j+1|k},x_{i,k+j|k},u_{i,k+j|k}),
\end{align*}

using the object-level density $\rho_b$ of
\eqref{eq:boxDensityConstruction} and the robot-level density $\rho_i$ of
\eqref{eq:robotDensityConstruction}, respectively. These are exactly
the density-safety constraints $\mathcal D_{b,j}\geq0$ and
$\mathcal D_{i,j}\geq0$ imposed in the object-level MPC
\eqref{eq:objectLevelMPC} and robot-level MPC \eqref{eq:robotLevelMPC}.
For robot $i$, the density is evaluated using the predicted contact
target at each stage and the neighboring-robot safety parameters
updated at time $k$; these dependencies are suppressed for brevity.

\begin{remark}
\label{rem:closed_loop}
\textcolor{black}{Theorem~\ref{thm:discrete_density_condition} converts the
continuous-time density-divergence condition used in
\cite{moyalan2024robustcdf} into a one-step discrete-time condition. At each
prediction stage, the input is held constant over the finite sampling
interval $T_s$ (zero-order hold), defining the one-step flow map.
Integrating the continuous-time condition over this interval yields the
stated Perron--Frobenius inequality. Thus, the fixed input in the theorem
is local to each one-step map; it does not assume a constant input over
the full closed-loop trajectory. In MPC, the density condition is enforced pointwise using the input chosen at each stage and is re-imposed whenever the object-level and robot-level
controllers re-solve.}
\end{remark}


\subsection{Object-Level Density MPC}

The object-level MPC in  \eqref{eq:objectLevelMPC} optimizes the contact-force sequence
$\mathbf{F}_{c,k}=\{f_{c,k+j|k}\}_{j=0}^{N_b-1}$ with cost
\begin{equation}
\begin{aligned}
J_b(\mathbf{F}_{c,k};x_b(k))
&=\sum_{j=0}^{N_b-1}
\ell_b(x_{b,k+j|k},f_{c,k+j|k})\\
&\quad+V_b^f(x_{b,k+N_b|k}),
\end{aligned}
\label{eq:boxMPCcost}
\end{equation}
where $\ell_b$ and $V_b^f$ are the stage and terminal
costs. Concretely,
\begin{equation*}
\begin{aligned}
\ell_b(x_{b,k+j|k},f_{c,k+j|k})
&=R_{fc}\|f_{c,k+j|k}-f_c^{\mathrm{nom}}\|^2\\
&\;+R_{\mathrm{diff}}(\Delta f_{c,k+j|k}-\Delta f_c^{\mathrm{nom}})^2\\
&\;+R_{\dot{fc}}\|f_{c,k+j|k}-f_{c,k+j-1|k}\|^2\\
&\;+R_{\dot{\mathrm{diff}}}(\dot\Delta f_{c,k+j|k})^2
+\tfrac{1}{2}w_sC_{k+j|k}^2,
\end{aligned}
\end{equation*}
with $\Delta f_c:=f_{c,1}-f_{c,2}$, $C_{k+j|k}\geq0$ the density-rate
slack of \eqref{eq:densitySafetyConstraint}. The weights $R_{fc},R_{\mathrm{diff}},R_{\dot{fc}},R_{\dot{\mathrm{diff}}},w_s$ and other MPC related paramteres are listed in Table~\ref{tab:mpc_params}. Here, $\mathbf{1}_n$ denotes the $n$-dimensional vector of ones. The box-level MPC is presented below.
Note that there is no state-tracking term used at
the object-level MPC, and the goal convergence is enforced by the density constraint
$\mathcal{D}_{b,j}\geq0$, and the cost $\ell_b$ regulates
only the contact force sequence.

\par\smallskip
\noindent
\begin{minipage}{\linewidth}
\setlength{\abovedisplayskip}{3pt}
\setlength{\belowdisplayskip}{3pt}
\setlength{\abovedisplayshortskip}{3pt}
\setlength{\belowdisplayshortskip}{3pt}
\setlength{\jot}{1pt}
\hrule height 0.4pt
\vspace{2mm}
\textbf{Box-Level Density MPC:}\\
\vspace{-2mm}
\begin{equation}
\begin{aligned}
\min_{\mathbf{F}_{c,k}}\quad
& J_b(\mathbf{F}_{c,k};x_b(k))\\
\text{s.t.}\quad
& x_{b,k+j+1|k}
  =F_b(x_{b,k+j|k},f_{c,k+j|k};d_{1:N}),\\
& \mathcal{D}_{b,j}\geq0,\\
& 0\leq f_{c,i,k+j|k}\leq f_{c,\max},\\
& x_{b,k|k}=x_b(k).
\end{aligned}
\label{eq:objectLevelMPC}
\end{equation}
\noindent
$j=0,\ldots,N_b-1$, \quad $i\in\mathcal{I}_N$.
\par\smallskip
\hrule height 0.4pt
\end{minipage}
\par\smallskip

The first optimal force vector
$f_c(k)=f_{c,k|k}^{\star}$ supplies the robot interaction-force
references. The predicted object trajectory determines
the moving contact targets
\begin{equation}
p_{c,i,k+j|k}
=p_{b,k+j|k}^{\star}
+R(\psi_{b,k+j|k}^{\star})[0,d_i]^\top.
\label{eq:predictedContactPoint}
\end{equation}

\subsection{Robot-Level Density MPC}

Robot $i$ optimizes
$\mathbf{U}_{i,k}=\{u_{i,k+j|k}\}_{j=0}^{N_i-1}$ with cost
\begin{equation}
\begin{aligned}
J_i(\mathbf{U}_{i,k};x_i(k))
&=\sum_{j=0}^{N_i-1}
\ell_i(x_{i,k+j|k},u_{i,k+j|k},p_{c,i,k+j|k})\\
&\quad+V_i^f(x_{i,k+N_i|k},p_{c,i,k+N_i|k}),
\end{aligned}
\label{eq:robotMPCcost}
\end{equation}
where $\ell_i$ and $V_i^f$ are the stage and terminal
costs for the assigned contact targets. The stage and terminal costs are
\begin{equation*}
\begin{aligned}
\ell_i
&=\|x_{i,k+j|k}-x_{i,\mathrm{ref},k+j|k}\|_{Q_i}^{2}
+\|u_{i,k+j|k}\|_{R_i}^{2},\\
V_i^f
&=\|x_{i,k+N_i|k}-x_{i,\mathrm{ref},k+N_i|k}\|_{Q_i}^{2},
\end{aligned}
\end{equation*}
where $x_{i,\mathrm{ref},k+j|k}$ combines the robot pose reference
associated with the assigned contact location $p_{c,i,k+j|k}$
and the nominal locomotion references for the remaining state
components. The terminal cost uses the same weight $Q_i$
as the stage cost. The robot-level MPC is presented below.
\par\smallskip
\noindent
\begin{minipage}{\linewidth}
\setlength{\abovedisplayskip}{3pt}
\setlength{\belowdisplayskip}{3pt}
\setlength{\abovedisplayshortskip}{3pt}
\setlength{\belowdisplayshortskip}{3pt}
\setlength{\jot}{1pt}
\hrule height 0.4pt
\vspace{2mm}
\textbf{Robot-Level Density MPC:}\\
\vspace{-2mm}
\begin{equation}
\begin{aligned}
\min_{\mathbf{U}_{i,k}}\quad
& J_i(\mathbf{U}_{i,k};x_i(k))\\
\text{s.t.}\quad
& x_{i,k+j+1|k}
  =F_i(x_{i,k+j|k},u_{i,k+j|k}),\\
& \mathcal{D}_{i,j}\geq0,\\
& u_{i,k+j|k}\in\mathcal{U}_i,\\
& x_{i,k|k}=x_i(k).
\end{aligned}
\label{eq:robotLevelMPC}
\end{equation}
\noindent
$j=0,\ldots,N_i-1$.
\par\smallskip
\hrule height 0.4pt
\end{minipage}
\par\smallskip

\begin{remark}[Active Constraint Set for $N>2$ agents]\label{rem:nearestTeammate}
For the two-robot team, \eqref{eq:agentUnsafeSet} involves a single
other robot and is instantiated directly. For the three-robot team, the
robot-level NMPC exposes one inter-agent avoidance parameter slot per
robot. Rather than instantiate all $N-1$ pairwise constraints in
\eqref{eq:agentUnsafeSet} simultaneously, each robot's controller is
given the position of only its nearest other robot at that instant,
re-selected every control step. Every teammate state remains available
under Assumption~\ref{as:info}. Only the active \emph{constraint} is
restricted to the nearest teammate. Instantiating every pairwise
constraint in \eqref{eq:agentUnsafeSet} simultaneously for larger teams
is a straightforward generalization.
\end{remark}

The first optimal input
$u_i(k)=u_{i,k|k}^{\star}$ is applied through the
whole-body controller. At each sampling instant, the controller measures the
object and robot states, solves the object-level MPC,
and generates interaction-force references and predicted
contact targets. Each robot then updates its obstacle
and inter-agent safety parameters, and solves its local
MPC.

\begin{table}[t]
\centering
\caption{MPC parameters for object-level and robot-level controllers.}
\label{tab:mpc_params}
\scriptsize
\renewcommand{\arraystretch}{0.85}
\setlength{\tabcolsep}{3pt}
\begin{tabular*}{\columnwidth}{@{\extracolsep{\fill}}lc@{}}
\toprule
Parameter & Value \\
\midrule
\multicolumn{2}{@{}l}{\textbf{Object-Level Density MPC (box)}} \\
Prediction horizon $N_b$                        & $15$ \\
Discretization interval $T_s$                   & $0.05$~s ($20$~Hz) \\
Box--ground friction $\mu$                      & $0.5$ \\
Contact-force bounds $[f_{c,\min},f_{c,\max}]$  & $[0,45]$~N \\
Density exponent $\alpha$                       & $1.0$ \\
Orientation weight $w_\theta$                   & $0.3$ \\
Density cap $\rho_{\max}$                       & $20.0$ \\
\midrule
\multicolumn{2}{@{}l}{\textbf{Object-Level Cost Weights}} \\
Force tracking $R_{fc}$                         & $1\times10^{-2}$ \\
Force-differential tracking $R_{\mathrm{diff}}$ & $5\times10^{-1}$ \\
Force-rate $R_{\dot{fc}}$                       & $1\times10^{-3}$ \\
Differential-rate $R_{\dot{\mathrm{diff}}}$     & $2\times10^{-2}$ \\
Density-constraint slack penalty                & $1\times10^{4}$ \\
\midrule
\multicolumn{2}{@{}l}{\textbf{Robot-Level Density NMPC (per robot)}} \\
Prediction horizon $N_i$                        & $12$ \\
Discretization interval                         & $0.02$~s \\
Ground friction $\mu$                           & $0.5$ \\
Per-leg GRF bounds $[f_{\min},f_{\max}]$        & $[0,147.3]$~N \\
\midrule
\multicolumn{2}{@{}l}{\textbf{Robot-Level Cost Weights}} \\
State tracking $\mathrm{diag}(Q_i)$ &
\makecell[l]{
$(0,0,1500,\,200\mathbf{1}_3^\top,\,500,500,0,$\\
$\quad 20,20,50,\,300\mathbf{1}_3^\top,\,50,\,10\mathbf{1}_5^\top)$
} \\
Input regulation $\mathrm{diag}(R_i)$ &
$(10^{-4},10^{-4},10^{-5},\,10^{-3}\mathbf{1}_3^\top)$ \\
\midrule
\multicolumn{2}{@{}l}{\textbf{Physical Parameters}} \\
Box mass, dimensions                            & $5$~kg, $0.5\!\times\!1.2\!\times\!0.5$~m \\
Robot mass (Unitree Go2)                        & $15.019$~kg \\
\bottomrule
\end{tabular*}
\end{table}


\begin{figure*}[t]
\centering
\begin{subfigure}[t]{\textwidth}
    \centering
    \includegraphics[width=\linewidth]
        {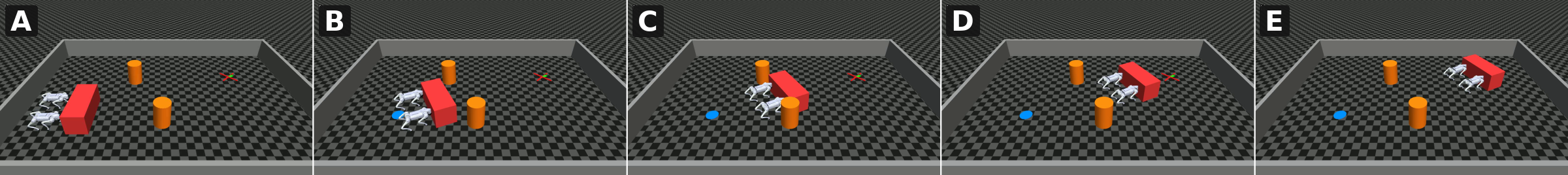}
    \caption{Two quadrupeds transporting a rectangular box.}
\end{subfigure}

\smallskip
\begin{subfigure}[t]{\textwidth}
    \centering
    \includegraphics[width=\linewidth]
        {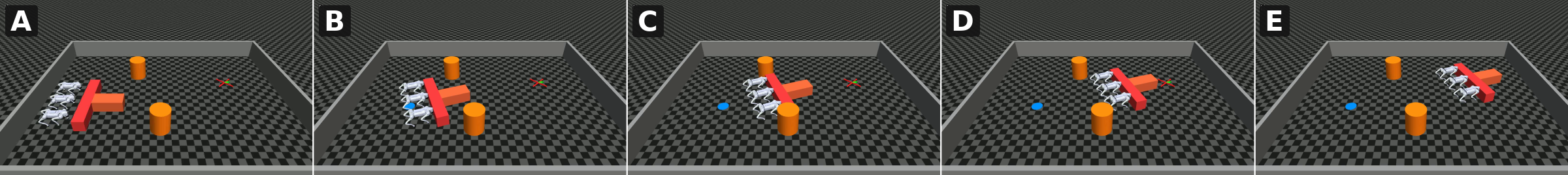}
    \caption{Three quadrupeds transporting a T-shaped object.}
\end{subfigure}
\caption{Collaborative manipulation using the proposed density MPC framework evaluated in MuJoCo, from the initial
configuration (A) to the goal configuration (E).}
\label{fig:filmstrip}
\end{figure*}

\section{Multi Robot Collaborative Transport}\label{sec:main_results}

We evaluate the proposed framework in MuJoCo using articulated
Unitree Go2 models with simulated robot--ground and robot--object
contacts. Two scenarios are considered: two robots pushing a
$0.5\times1.2\times0.5$-m, $5$-kg rectangular box through a
$1.90$-m gap, and three robots pushing a heavier T-shaped object
through a $2.50$-m gap. In both scenarios, two circular obstacles
form a narrow passage within a bounded workspace. The controller
combines the object-level and robot-level MPC formulations
in \eqref{eq:objectLevelMPC} and \eqref{eq:robotLevelMPC}, using
fixed contact offsets without a prescribed object trajectory. Performance is evaluated over 10 trials with random initial-pose
perturbations of up to $\pm0.05$~m in position and $\pm3^\circ$
in orientation. Evaluation metrics are goal attainment, simulated
completion time, and minimum recorded object--obstacle clearance,
measured from the rotated object footprint to the physical
obstacle boundaries.

Simulations run on an AMD Ryzen 9 5900X processor with a MuJoCo
integration step of $2$~ms. The object-level MPC is implemented
in CasADi and solved using IPOPT, with a $15$-step prediction
horizon and a $0.05$-s discretization interval. Each robot-level
MPC uses acados with sequential quadratic programming (SQP),
a $12$-step prediction horizon, and a $0.02$-s discretization
interval. The object- and robot-level controllers update at
$20$ and $100$~Hz, respectively, in simulation time.
Additional parameters are listed in Table~\ref{tab:mpc_params}.

Both teams successfully transport the object through the passage
and reach the desired pose
(Figs.~\ref{fig:filmstrip} and \ref{fig:comparison}).
Across 10 trials, the proposed controller achieves full task success,
with completion times of $38.6\pm0.8$~s for box pushing and
$28.5\pm1.0$~s for T-shape pushing. The representative trajectories
complete in $38.50$ and $28.92$~s, with minimum recorded
object--obstacle clearances of $0.02$ and $0.04$~m, respectively,
indicating separation throughout the sampled trajectories (Table~\ref{tab:baseline_compare}). The results demonstrate coordinated transport from a desired
object pose without a prescribed trajectory. The object-level MPC
\eqref{eq:objectLevelMPC} generates the motion, while the moving
contact locations \eqref{eq:predictedContactPoint} provide tracking
references for the robot-level controllers. Both tested team sizes
and object geometries use the same hierarchical formulation
with scenario-specific parameters.

With solver reuse and state warm-starting, Successful object-level MPC solves took
$15.55 \pm 2.39$~ms for box pushing and $16.98 \pm 2.25$~ms
for T-shape pushing, across 449 and 377 successful solves,
respectively. The corresponding per-robot MPC computation times
were $1.54 \pm 0.26$~ms and $1.47 \pm 0.25$~ms for the
two-robot team, and $1.56 \pm 0.24$~ms, $1.50 \pm 0.24$~ms,
and $1.50 \pm 0.23$~ms for the three-robot team. The measured computation times support the target implementation
with object-level MPC at 20~Hz and robot-level MPC at 100~Hz.

\begin{figure}
    \centering
    \includegraphics[width=\linewidth]{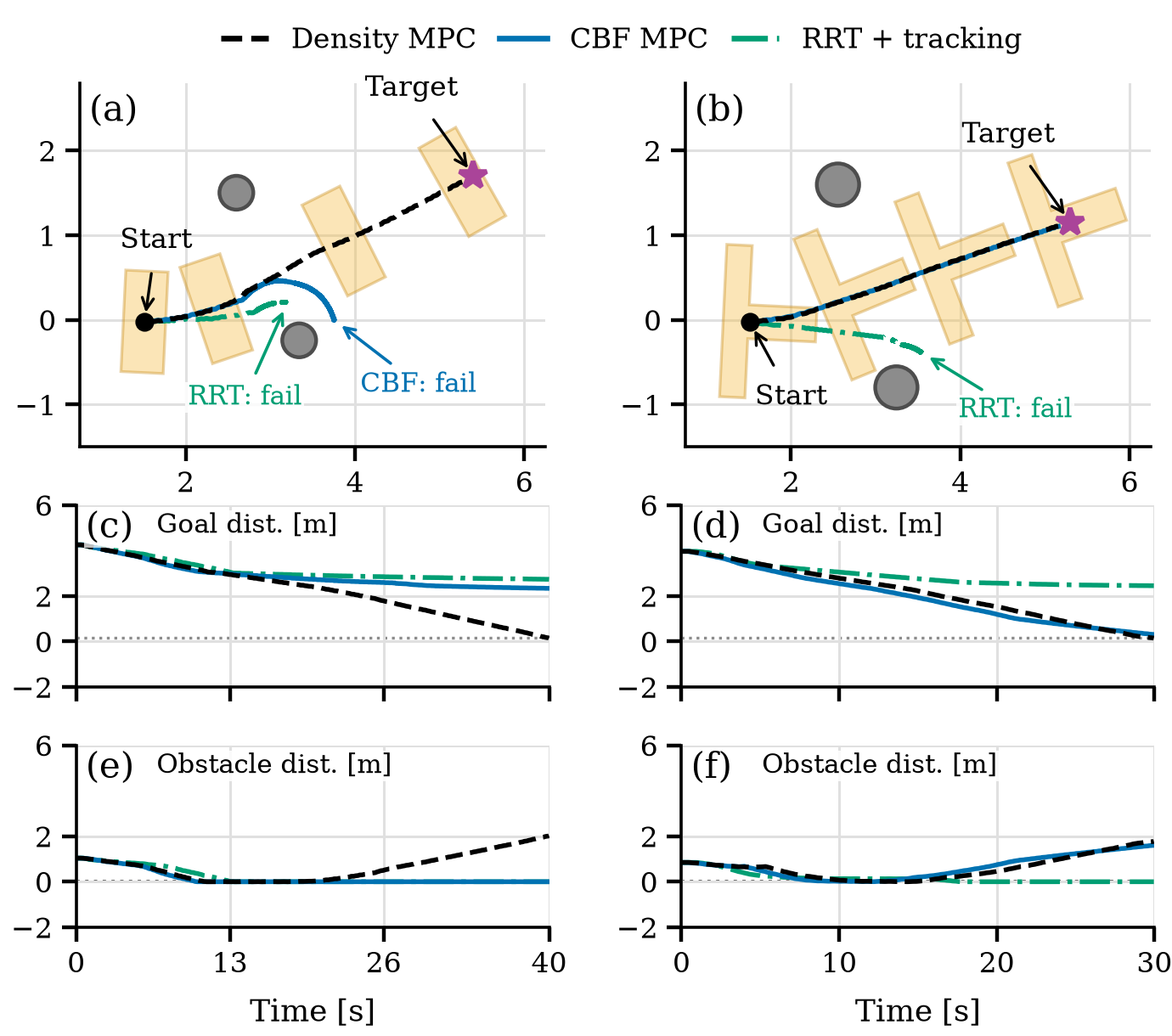}
    \caption{Comparison of Density MPC, CBF MPC, and RRT with tracking for box (left) and T-shaped object (right) pushing: (a,b) object trajectories; (c,d) distance to goal; (e,f) minimum object–obstacle clearance.}
    \label{fig:comparison}
\end{figure}

\begin{table}[t]
\centering
\caption{Baseline comparisons.}
\label{tab:baseline_compare}
\scriptsize
\setlength{\tabcolsep}{3pt}
\renewcommand{\arraystretch}{0.8}
\begin{tabular*}{\columnwidth}{@{\extracolsep{\fill}}lccc@{}}
\toprule
Method & Success & Completion Time (s) & Min. Clearance (m) \\
\midrule
\multicolumn{4}{@{}l}{\textit{Box pushing ($n{=}2$)}} \\
\textbf{Density MPC} & $10/10$ & $38.6\pm0.8$ & $0.006\pm0.009$ \\
CBF MPC~\cite{zeng2021mpccbf} & $5/10$ & $48.0\pm0.7$ & $0.000\pm0.002$ \\
RRT*+track~\cite{karaman2011sampling} & $0/10$ & -- & $-0.001\pm0.000$ \\
\midrule
\multicolumn{4}{@{}l}{\textit{T-shape pushing ($n{=}3$)}} \\
\textbf{Density MPC} & $10/10$ & $28.5\pm1.0$ & $0.040\pm0.017$ \\
CBF MPC & $10/10$ & $31.1\pm0.8$ & $0.049\pm0.022$ \\
RRT*+track & $0/10$ & -- & $-0.001\pm0.000$ \\
\bottomrule
\end{tabular*}
\end{table}

\textbf{Comparison with Baselines.} We compare density MPC with CBF MPC and RRT* based tracking MPC using the same dynamics and MPC setup. CBF MPC replaces the object-level density constraint with obstacle CBF constraints in the object-level and robot-level MPC (decay parameter, $\gamma=0.7$, with all other parameters fixed), while RRT* plans through inflated obstacles and provides a pure-pursuit reference with density avoidance disabled. RRT* computes a single collision-free path offline (step size $0.4$~m, goal-sampling bias $0.15$, up to $4000$ iterations, goal tolerance $0.25$~m, obstacles inflated by the box footprint plus a $0.05$~m safety pad), smoothed with random shortcutting; the resulting path is converted into a moving reference via pure-pursuit tracking (lookahead $1.0$~m) and executed by the same robot-level MPC controllers, with the density constraint disabled. In the deterministic runs, density MPC reduces completion time relative to CBF MPC by $22.5\%$ for box pushing and $6.3\%$ for T-shape pushing, while maintaining comparable or larger clearance (Table~\ref{tab:baseline_compare}). Density MPC succeeds in all trials, whereas CBF MPC succeeds in $5/10$ box-pushing trials and $10/10$ T-shape trials, indicating fragility only in the tighter box-pushing passage (Fig.~\ref{fig:comparison}). RRT* based tracking reduces the initial goal distance by only $38\%$ and $41\%$ before stalling because of heading drift and diminishing pushes, reflecting the mismatch between geometric reference tracking and push-only, force--torque-coupled actuation; it fails to reach the goal in both scenarios (see Fig.~\ref{fig:comparison}).

\textbf{Effect of Gait Synchronization.} To assess gait synchronization, we compare the default phase-locked
gaits with a half-cycle offset.
Push-force variability, measured
as $\mathrm{CV}_F=100\,\sigma_F/\mu_F$ (\%), increased from
$78.1\%$ to $84.5\%$, while mean step-to-step force change rose
from $13.09$ to $13.89$~N. Completion time changed little
($38.50$ vs.\ $37.90$~s), with no failures observed in these tests.
Thus, phase locking modestly improved force smoothness but was
not essential for successful pushing in this scenario. This suggests that phase locking helps coordinate contact loading, although broader tests are needed to determine its benefits.
This may be more consequential for a shared payload
carried atop the robots, where gait-induced relative motion
could affect payload stability.

\textbf{Limitations and Future Work.} 
Validation is limited to simulated box and T-shape pushing
with two and three robots, fixed contacts, and planar safety
envelopes. Larger teams require evaluation of computational
scalability, while arbitrary object geometries and complex
or dynamic environments require richer collision representations
beyond the current bounding-size margin. Future work will quantify minimum robot-envelope clearances,
account for estimation and geometric uncertainty near small
clearance margins, and validate contact-force tracking and
real-time execution on hardware.


\section{Conclusions}
\label{sec:conclusions}

We presented a hierarchical density-constrained MPC framework for multi-quadruped collaborative pushing without a prescribed object trajectory. Object-level contact-force optimization is coupled with robot-level tracking and density-based obstacle and inter-agent constraints. MuJoCo experiments with box and T-shape pushing demonstrated goal attainment with positive recorded object-footprint clearances. Compared with matched CBF MPC, density MPC reduced completion time by $22.5\%$ and $6.3\%$ in the two scenarios, while RRT* based tracking MPC failed to complete either task. The current validation is limited to simulated two- and three-robot teams, fixed contacts, and planar safety envelopes; future work will address computational scalability, dynamic environments, uncertainty, and real-time hardware validation.

\bibliographystyle{IEEEtran}
\bibliography{ref}

\end{document}